\documentclass{article}
\usepackage{iclr2027_conference,times}

\usepackage[utf8]{inputenc}
\usepackage[T1]{fontenc}
\usepackage{hyperref}
\usepackage{url}
\usepackage{booktabs}
\usepackage{amsfonts}
\usepackage{amsmath,amssymb}
\usepackage{amsthm}
\usepackage{nicefrac}
\usepackage{microtype}
\usepackage{xcolor}
\usepackage{graphicx}
\usepackage{subcaption}
\usepackage{wrapfig}
\usepackage{multirow}
\usepackage{float}
\usepackage{tikz}
\usetikzlibrary{positioning,arrows,arrows.meta}
\usepackage[noabbrev]{cleveref}

\theoremstyle{plain}
\newtheorem{proposition}{Proposition}
\theoremstyle{remark}
\newtheorem{remark}{Remark}

\usepackage{amsmath,amsfonts,bm}

\def\eqref#1{equation~\ref{#1}}

\def\1{\bm{1}}

\DeclareMathAlphabet{\mathsfit}{\encodingdefault}{\sfdefault}{m}{sl}
\SetMathAlphabet{\mathsfit}{bold}{\encodingdefault}{\sfdefault}{bx}{n}

\newcommand{\R}{\mathbb{R}}

\newcommand{\normlzero}{L^0}
\newcommand{\normlone}{L^1}
\newcommand{\normltwo}{L^2}

\def\ourca{ClasSAE}
\newcommand{\SoftTopK}{\operatorname{SoftTopK}}

\title{\ourca{}: Class-Aligned Sparse Autoencoders via Differentiable Feature-Class Affinity}

\iclrfinalcopy

\author{Jakub St\k{e}pie\'n$^{1}$,  Marcin Mazur$^{1}$, Jacek Tabor$^{1}$ Przemys{\l}aw Spurek$^{1,2}$\\
Jagiellonian University$^1$; 
IDEAS Research Institute$^2$ \\
\texttt{jaku8.stepien@student.uj.edu.pl} \\
} 

\begin{document}
\maketitle

\begin{abstract}

Sparse Autoencoders (SAEs) began as an unsupervised tool for decomposing neural representations into sparse, interpretable features, and are increasingly used not only for passive analysis but also for active interventions such as unlearning, bias mitigation, and concept editing. A central challenge for these editing and steering methods is reliably matching features to target concepts; most current approaches address this by computing post-hoc scores over an already-trained, frozen dictionary. We instead introduce ClasSAE, a novel method that both automatically assigns classes to features and guides the encoder toward class-separable representations during training. Specifically, we apply a differentiable top-$k$ operator to a trainable feature--class affinity matrix with per-feature budgets, coupling the features selected for each sample to the classes they are trained to represent. Because gradients flow through the selection of active features rather than only through their magnitudes, the encoder and the affinity matrix co-adapt rather than being fit in separate stages. The result is a dictionary that is both class-separable and class-annotated, with no need for post-hoc probing. We propose three variants for enforcing sparsity within this framework, which achieve comparable overall performance with slightly different trade-offs. Using CLIP ViT-L/14 embeddings on ImageNet, we show that the learned affinity matrix agrees closely with an independently estimated post-hoc feature-class matrix computed on held-out data. The model also supports direct class prediction from the encoder and affinity matrix alone, without a separately fitted classifier, and its more class-aligned encoder yields improved separation in Targeted Probe Perturbation evaluations. These results show that known concept annotations can be used not merely to locate relevant features in an SAE dictionary after training, but also to shape the dictionary itself during learning. Code is available at \url{https://github.com/St0pien/ClasSAE}.

\end{abstract}

\begin{figure}[h]
  \centering
  \includegraphics[width=0.85\textwidth]{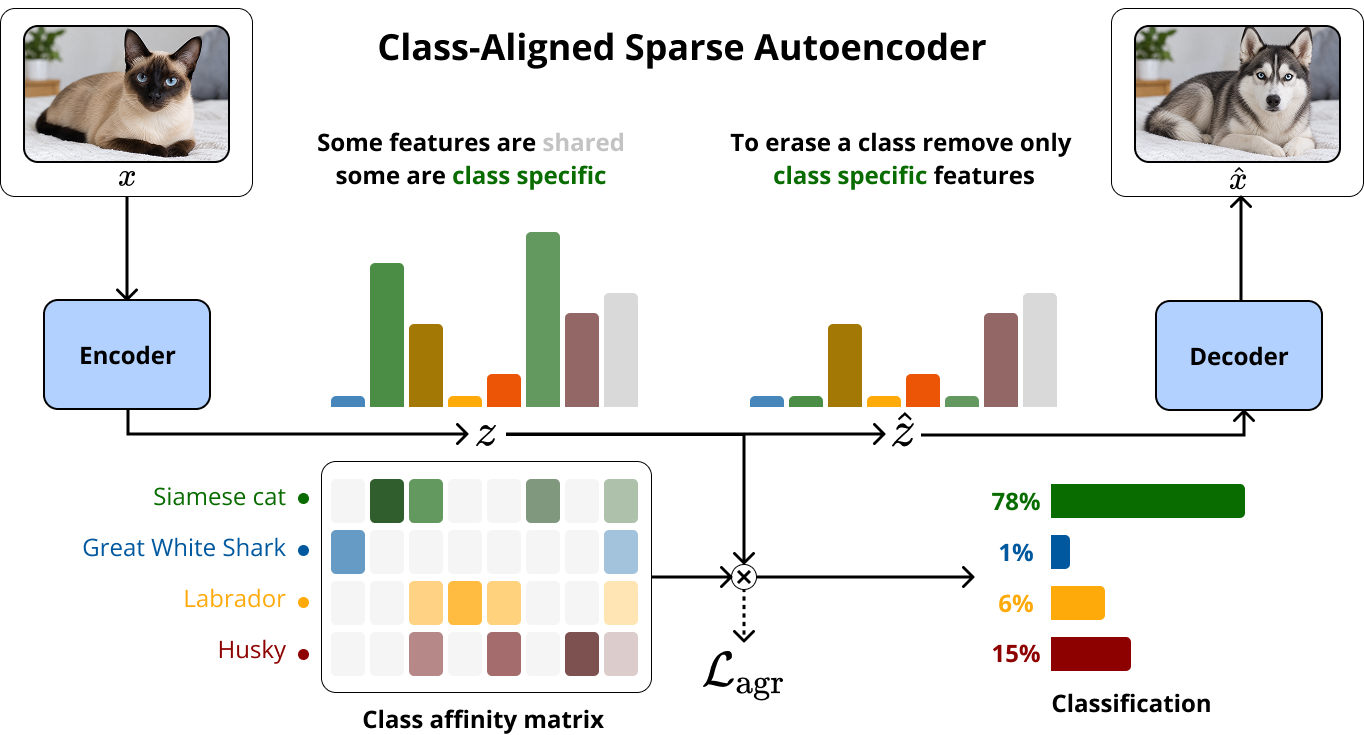}
  \caption{ClasSAE jointly learns a sparse dictionary and a class affinity matrix,
    enabling classifier-free class prediction and targeted editing: removing only the features exclusive to a class erases that class from the reconstruction while shared features
    are preserved.}
  \label{fig:teaser}
\end{figure}

\section{Introduction}\label{sec:intro}

Sparse autoencoders (SAEs) have become a leading tool for mechanistic interpretability, decomposing dense activations into sparse, putatively monosemantic features in a purely unsupervised way \citep{huben2023sparse, bricken2023monosemanticity, templeton2024scaling}. However, recent work questions whether this unsupervised decomposition is even well-defined: \citet{leask2025sparse} show that SAEs do not converge to a canonical set of ground-truth units, and different random seeds or hyperparameters yield dictionaries that disagree substantially with one another \citep{paulo2026sparse}.

Despite this, SAEs are widely used not only for analysis but also for \emph{intervention}, including activation steering, unlearning, and bias mitigation \citep{cywinski2025saeuron, ahsan2026can, cho2025toward}. All of these methods work by editing the small subset of features responsible for a target concept. This is exactly where the unsupervised nature of SAEs becomes a liability: since training has no access to the concepts practitioners care about, every such application must first solve a separate post-hoc matching problem, probing a frozen dictionary to infer which features correspond to which concept. This two-stage recipe inherits the instability noted above: a target concept may end up split across several features, entangled with other concepts, or absent from the dictionary altogether, and nothing during training prevents this.

We observe that in most of these intervention settings, the relevant concepts are not in fact unknown: unlearning a category or mitigating bias along a known attribute both presuppose that the practitioner knows in advance which classes should be disentangled. Rather than hoping that an unsupervised SAE will isolate these concepts and searching for them afterwards, we argue that this knowledge should directly shape training.

To this end, we introduce \textbf{ClasSAE}, an SAE that learns a class-annotated, class-separable dictionary directly during training rather than relying on post-hoc analysis. We provide a general visual overview of the proposed framework in \Cref{fig:teaser}. Our contributions are as follows:
\begin{itemize}
    \item \textbf{Automatic feature--class association.} We learn a feature--class affinity matrix jointly with the SAE, removing the need for post-hoc probing to identify which features correspond to which class.
    \item \textbf{Class-aligned encoder--decoder.} The encoder and decoder are trained so that the dictionary is class-separable by construction, rather than merely reconstruction-optimal with class structure emerging only incidentally.
    \item \textbf{Three sparsity variants under a shared framework.} We propose three ways of enforcing sparsity in the latent representation, which trade off reconstruction quality against class alignment.
    \item \textbf{A single joint training objective.} We combine these mechanisms into one end-to-end objective, so the dictionary and the affinity matrix co-adapt within one training process rather than being fit in separate stages. This general framework yields a single model with multiple applications: classification (\Cref{sec:eval-classifier}), interpretability (\Cref{fig:qualitative}), and targeted concept editing (\Cref{sec:tpp_unlearning}).
    \item \textbf{Experimental evaluation.} We assess the agreement of ClasSAE's affinity matrix with an independently estimated post-hoc feature--class matrix, its ability to support classifier-free class prediction, and its performance on Targeted Probe Perturbation, comparing against other SAE architectures and variants of our own method.
\end{itemize}

\section{Related Work}

\textbf{Sparse autoencoders and their failure modes.}
SAEs decompose activations into sparse linear combinations of learned directions \citep{sharkey2022taking,bricken2023monosemanticity,huben2023sparse} and scale to frontier models \citep{templeton2024scaling}. Architectural work has concentrated on the sparsity constraint: $k$-sparse and top-$k$ SAEs enforce an exact count \citep{makhzani2014ksparse,gao2024scaling}; Gated SAEs decouple \emph{which} features fire from \emph{how strongly} \citep{rajamanoharan2024gated}; JumpReLU uses a discontinuous threshold \citep{rajamanoharan2024jumping}; BatchTopK moves the budget from the sample to the batch, so individual samples may use more or fewer latents at fixed average sparsity \citep{bussmann2024batchtopk}; Matryoshka SAEs impose a nested hierarchy over dictionary prefixes \citep{bussmann2025learning, zaigrajew2025interpreting}. \citet{lee2025qo} instead enforce approximate quasi-orthogonality among dictionary elements, which enables a Top-AFA activation function that removes the $k$ hyperparameter entirely by choosing the number of active features whose reconstruction most closely matches the input's $\normltwo$ norm.

\begin{wrapfigure}{r}{0.6\textwidth}
\vspace{-0.5cm}
    \centering
    \includegraphics[width=0.65\textwidth]{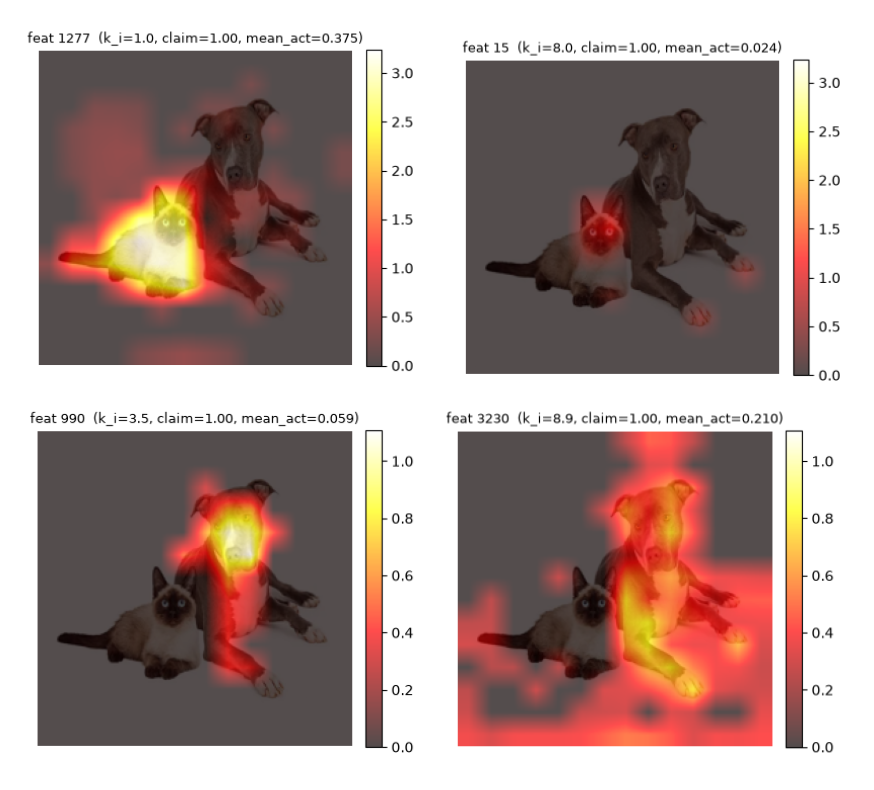}
    \caption{Feature activation heatmaps for \ourca{} trained on DINOv3 patch tokens, using whole-image class labels only. \textbf{Top:} features claimed by ``Siamese cat, Siamese.'' \textbf{Bottom:} features claimed by ``American Staffordshire terrier, Staffordshire terrier, American pit bull terrier, pit bull terrier.'' Each panel reports the feature's budget $k_i$, its claim strength for the labeled class, and its mean activation. Details in \Cref{app:dino}.}
    \label{fig:qualitative}
    \vspace{-0.4cm}
\end{wrapfigure}

A parallel line documents what SAEs get wrong: SAEs do not reliably decompose concepts into canonical units \citep{leask2025sparse}, and apparently monosemantic latents are \emph{absorbed} into their children \citep{chanin2025absorption}; dictionaries are unstable across seeds, motivating constrained parameterizations \citep{fel2025archetypal}; much of the residual is linearly predictable ``dark matter'' \citep{engels2025darkmatter}; and the i.i.d.\ prior of per-token SAEs conflicts with the temporal structure of language \citep{bhalla2026temporal,lubana2026priors}. SAEBench \citep{pmlr-v267-karvonen25a} consolidates several of these diagnostics, one of which we adopt in our experiments (see \Cref{sec:exp}). We hypothesize that these failure modes stem, at least in part, from the absence of any external signal guiding the dictionary toward useful structure; the class-alignment mechanism we introduce in this work is one candidate for such a signal, though our focus here is on evaluating its practical, measurable benefits rather than directly testing this broader hypothesis.

\textbf{Automatic feature interpretation and labeling.} A separate line automates the post-hoc labeling step itself. \citet{bricken2023monosemanticity} establish the basic dictionary-learning setup that makes such labeling meaningful in the first place. \citet{paulo2025automatically} generate natural-language feature explanations at scale using LLMs and, notably, evaluate them via intervention rather than input examples alone. \citet{gur-arieh-etal-2025-enhancing} show that input-only explanations are incomplete, since a feature's meaning also depends on its effect on the output, and propose output-centric descriptions from steering/unembedding, finding input- and output-centric signals complementary. SALVE \citep{flovik2026salve} similarly correlates SAE features with visual concepts, using Grad-CAM \citep{DBLP:journals/corr/SelvarajuDVCPB16} and activation maximization to explain individual features. All of these approaches treat labeling as a separate step applied to an already-trained, frozen dictionary; \ourca{} instead produces class annotations as a direct byproduct of training, with no explanation pipeline required.

\textbf{SAEs as intervention mechanisms.} SAEs are increasingly used to directly edit model behavior. SAEuron \citep{cywinski2025saeuron} trains a BatchTopK SAE and scores features to find those tied to a target concept, then ablates or negatively scales them to unlearn it from diffusion models. SAEmnesia \citep{cassano2026saemnesia} is closest in spirit to \ourca{}: a two-stage scheme first trains an unsupervised SAE, then collapses each class onto a single latent coordinate. \ourca{} differs on both counts, using one joint objective rather than two stages, and a per-feature affinity matrix rather than one coordinate per class, so features can be shared across classes rather than duplicated. This avoids wasting capacity and, for unlearning, lets us ablate only class-exclusive features while shared features are preserved for retain accuracy. \citet{cassano2026look} find SAEs good at locating concepts in diffusion models but show that direct feature ablation pushes activations off-manifold, causing artifacts, and use the located features to guide patch-swapping instead. CRISP \citep{ashuach2026crisp} trains a LoRA adapter \citep{hu2022lora} guided by SAE features to suppress a target concept while preserving unrelated ones in LLM unlearning. \citet{ahsan2026can} apply the locate-then-ablate recipe to detect and mitigate racial bias in healthcare LLMs via probing. Unlike these approaches, which first train an SAE and then annotate its features, \ourca{} folds both goals into a single training objective.

\textbf{Concept supervision and differentiable selection.}
Relaxations of sorting and top-$k$ include Gumbel reparameterizations \citep{jang2017categorical} and optimal-transport formulations \citep{xie2020differentiable}; we use LapSUM \citep{lapsum2025}, which accepts a budget and returns weights summing exactly to it, giving a differentiable selection signal for the per-feature $k_i$ in \ourca{}. Concept activation vectors and interpretable basis decompositions map latent spaces to human-named ideas \citep{kim2018interpretability,zhou2018interpretable}, later automated \citep{ghorbani2019towards,bykov2023labeling}, while Concept Bottleneck and Concept Embedding Models enforce such representations during training \citep{koh2020concept,espinosa2022concept}. \ourca{} sits between these traditions: the encoder and decoder stay unsupervised at inference, and only a budgeted auxiliary head carries class semantics.

\section{Method}\label{sec:method}

\paragraph{Preliminaries} Let $x\in\R^{n}$ be an activation extracted from a pretrained network and $d\gg n$ the dictionary size. With a pre-encoder bias $b_{\mathrm{pre}}\in\R^{n}$, encoder weights $W_{\mathrm{enc}}\in\R^{d\times n}$, encoder bias $b_{\mathrm{enc}}\in\R^{d}$ and decoder $W_{\mathrm{dec}}\in\R^{n\times d}$,
\begin{equation}
z=\mathrm{ReLU}\!\left(W_{\mathrm{enc}}(x-b_{\mathrm{pre}})+b_{\mathrm{enc}}\right)\in\R_{\ge0}^{d},
\qquad
\hat x = W_{\mathrm{dec}}\,f(z)+b_{\mathrm{pre}},
\label{eq:sae}
\end{equation}
where $f$ is the sparsification and the columns of $W_{\mathrm{dec}}$ form the dictionary. Training minimizes $\|x-\hat x\|_2^2$ under a sparsity mechanism: an $\normlone$ penalty, a threshold, or a hard top-$k$ that fixes $\|f(z)\|_0=k$ for every input.

\paragraph{Soft top-$k$} We use the differentiable operator of \citet{lapsum2025}, which we state explicitly since the formal definition matters below. Let $\psi$ be a smooth, strictly increasing CDF (the Laplace CDF in \citealp{lapsum2025}) and $\alpha>0$ a temperature. For $v\in\R^{m}$ and a budget $\beta\in(0,m)$ there is a unique threshold $t=t(v,\beta,\alpha)$ solving $\sum_{j}\psi\!\left((v_j-t)/\alpha\right)=\beta$, and
\begin{equation}
\SoftTopK(v,\beta,\alpha)_j \;=\; \psi\!\left(\frac{v_j-t}{\alpha}\right)\in(0,1),
\qquad
\sum_{j=1}^{m}\SoftTopK(v,\beta,\alpha)_j=\beta .
\label{eq:softtopk}
\end{equation}
Two properties are used throughout: the budget $\beta$ need not be an integer, and as $\alpha\to 0$ the output concentrates on the $\lceil\beta\rceil$ largest entries. Both $t$ and the outputs are differentiable in $v$, $\beta$ and $\alpha$.

\paragraph{A class affinity matrix: per-feature annotation}\label{sec:ca}
We begin by constructing an automatic feature annotation mechanism. Let $C$ be the number of classes and $\Lambda\in\R^{d\times C}$ a trainable matrix of \emph{class logits}: $\Lambda_{ic}$ says how much coordinate $i$ ``wants'' class $c$. Raw logits mean nothing on their own; what gives them meaning is that claims are scarce. We fix a global association budget $K_{\mathrm{tot}}=\rho\,d$ (on average $\rho$ classes per coordinate) and let a trainable vector $\kappa\in\R^{d}$ split it:
\begin{equation}
k \;=\; K_{\mathrm{tot}}\cdot\mathrm{softmax}(\kappa)\in\R_{>0}^{d},\qquad \textstyle\sum_i k_i=K_{\mathrm{tot}},
\qquad
M_i=\SoftTopK(\Lambda_i,k_i,\alpha_M)\in(0,1)^{C}.
\label{eq:budget}
\end{equation}
Row-wise, $M_i$ answers ``which classes does this feature serve''; column-wise, ``which features serve this class''. Nothing forces the $k_i$ to be equal.\footnote{We clip $k_i\le C$ for numerical safety, since $\SoftTopK$'s budget cannot exceed the dimension it selects over; in practice a well-trained model never approaches this bound, as doing so would waste budget with no benefit to any row's agreement score.} Since $M$ depends only on parameters, \Cref{eq:budget} is one batched operator call per step on a $d\times C$ matrix, i.e.\ $16.4$M entries at $d=16384$, $C=1000$, which is negligible beside the encoder.

\paragraph{Three sparsity variants}\label{sec:sparsity}
The sparsity mechanism is a central design choice in any SAE. We propose three alternatives, each achieving comparable performance with different trade-offs, discussed in \Cref{sec:exp}. Variants (i) and (ii) share the same forward-pass budget \citep{bussmann2024batchtopk}, zeroing every coordinate outside the top $\bar k \cdot B$ activations in a batch of size $B$, where $\bar k \ll d$ is the target $\normlzero$, and differ only in how gradients reach the encoder; variant (iii) uses no thresholding operator at all.

\emph{(i) Soft-gradient BatchTopK.} We replace the identity straight-through estimator with a gradient borrowed from a soft surrogate,
\begin{equation}
f'(z) = p(x)\odot z, \qquad p(x)=\SoftTopK\big(z,\hat k(x),\alpha_p\big)\in(0,1)^d, \qquad \hat k(x)=\|f(z)\|_0,
\end{equation}
and use $f(z) - \text{stop\_grad}\big(f'(z)\big) + f'(z)$ as the representation: the forward pass is exactly hard BatchTopK, while the backward pass flows through the soft ranking, giving gradient signal through \emph{selection} rather than only through the surviving coordinates' scale.

\emph{(ii) Hard BatchTopK.} The standard estimator: $f(z)$ forward, identity gradient on active coordinates, zero elsewhere.

\emph{(iii) $\normlone$-relaxed.} No thresholding operator at all; sparsity is encouraged only via an $\normlone$ penalty added to the loss, with the encoder otherwise dense.

\begin{remark}
    Variant (i) may look like it pays for an extra SoftTopK call that (ii) and (iii) avoid. It does not: every variant must compute $p(x)$ regardless \footnote{Note that $p(x)$ can be computed for any SAE's $z$; we use this fact in our experiments involving post-hoc estimated matrices in \Cref{sec:honesty}.} , since the agreement loss explained in the next paragraph requires it.
\end{remark}

\paragraph{Coupling the two budgets}\label{subsec:agreement}
Both $p(x)$ and $M$ are budgeted selections, so we ask them to agree. We use the \emph{selection mass}, not the activation magnitude, normalizing $p$ by its own budget:
\begin{equation}
\pi(x)=\frac{p(x)}{\hat k(x)}\in\Delta^{d-1},
\qquad
s_c(x)=\left(M^{\!\top}\pi(x)\right)_c=\frac{1}{\hat k}\sum_{i}p_i M_{ic}\in[0,1].
\label{eq:score}
\end{equation}
By the defining property of \Cref{eq:softtopk}, $\pi(x)$ lands on the simplex exactly, with no $\varepsilon$ and no dependence on the scale of the activations. The agreement loss over a batch $\{(x^{(b)},c_b)\}_{b\le B}$ is
\begin{equation}
\mathcal{L}_{\mathrm{agr}}=-\frac{1}{B}\sum_{b=1}^{B}s^{(b)}_{c_b},
\qquad
\mathcal{L}(x)=\|x-\hat x\|_2^2+\gamma\,\mathcal{L}_{\mathrm{aux}}+\mu\,\mathcal{L}_{\mathrm{agr}}.
\label{eq:agr}
\end{equation}
where $\mathcal{L}_{\mathrm{aux}}$ is a standard dead-latent loss proposed by \citet{gao2024scaling}. For variant (iii) we also add $\normlone$ penalty as mentioned earlier. Gradients flow both ways: forward into the annotation ($\Lambda$: which classes, $\kappa$: how many) and backward through $p$ into the encoder. Without the second path, this would be post-hoc labeling with extra steps. \Cref{fig:method} visualizes this mechanism.

\begin{remark}[Why $p$ and not $z$]\label{rem:pnotz}
The construction is about which features a sample uses, not how strongly they fire. Magnitudes are set by reconstruction and decoder-column norms, so a score built on $z$ lets one large activation dominate a sample that selected $\hat k=40$ features. With $p$, each selected feature contributes $\approx1$ and $s_c$ is exactly the claimed fraction of the selection; gradients then reach the encoder through the ranking that $\SoftTopK$ differentiates, pushing class-consistent features \emph{up the order} rather than competing with reconstruction over magnitudes.
\end{remark}

\begin{remark}[Precision, not recall]\label{rem:precision}
In the hard limit both objects are sets and $s_c=|S(x)\cap F_c|/|S(x)|$ is a \emph{precision}, maximal as soon as $S(x)\subseteq F_c$: a class may own $40$ features while any sample uses $8$. The soft version inherits this, since $\sum_i p_i=\hat k$ gives $1-s_c=\hat k^{-1}\sum_i p_i(1-M_{ic})$, a condition on the selected coordinates only. Normalizing by $|F_c|$ instead would make $s_c$ a recall, force every sample to fire everything its class owns, reward larger $\hat k$, and suppress the adaptivity we want to preserve.
\end{remark}


\definecolor{siamesegreen}{HTML}{096C00}   
\definecolor{sharkblue}{HTML}{00589F}      
\definecolor{labradoramber}{HTML}{FDB800}  
\definecolor{huskyred}{HTML}{D90000}       

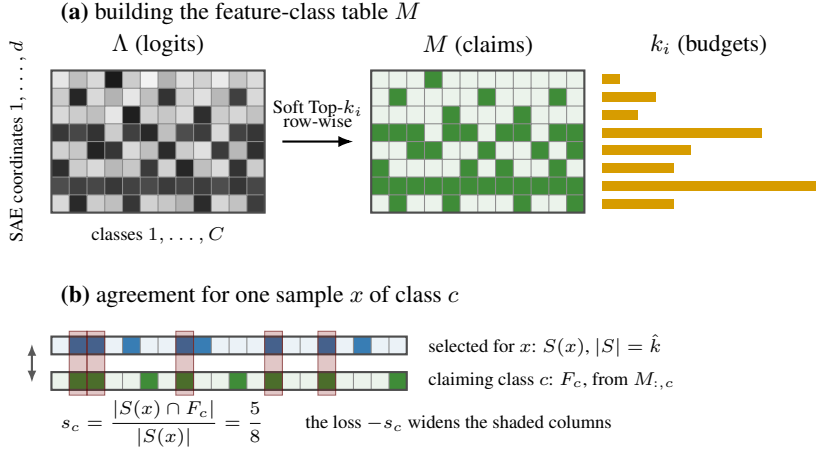
\begin{figure}[t]
\centering
\begin{tikzpicture}[x=0.235cm, y=0.235cm, font=\small]

\node[anchor=south west] at (0,2.1) {\textbf{(a)} building the feature-class table $M$};

\foreach \r/\row in {%
  1/{10,25,15,90,20,5,30,12,18,8,22,14},
  2/{20,85,15,30,10,25,80,18,12,28,75,15},
  3/{15,20,25,10,88,18,22,14,82,20,16,24},
  4/{78,72,80,20,75,84,25,70,76,15,74,79},
  5/{18,22,86,79,15,25,83,20,12,77,28,81},
  6/{82,20,15,28,22,77,18,85,25,14,80,20},
  7/{70,75,72,78,74,71,77,73,76,70,79,72},
  8/{15,84,20,25,79,18,22,30,83,16,24,76}}
  {\foreach \v [count=\c] in \row
     {\fill[black!\v] ({\c-1},{-\r}) rectangle ++(1,1);}}
\draw[black!45,step=1] (0,-8) grid (12,0);
\draw[black!70,thick] (0,-8) rectangle (12,0);
\node[anchor=south,font=\footnotesize] at (6,0.25) {$\Lambda$ (logits)};
\node[rotate=90,anchor=south,font=\scriptsize] at (-0.9,-4) {SAE coordinates $1,\dots,d$};
\node[anchor=north,font=\scriptsize] at (6,-8.3) {classes $1,\dots,C$};

\draw[-{Latex[length=2mm]},thick] (13,-4) -- (17,-4);
\node[anchor=south,font=\scriptsize,align=center] at (15,-3.7) {Soft Top-$k_i$\\[-1mm] row-wise};

\foreach \r/\row in {%
  1/{8,8,8,78,8,8,8,8,8,8,8,8},
  2/{8,78,8,8,8,8,78,8,8,8,78,8},
  3/{8,8,8,8,78,8,8,8,78,8,8,8},
  4/{78,78,78,8,78,78,8,78,78,8,78,78},
  5/{8,8,78,78,8,8,78,8,8,78,8,78},
  6/{78,8,8,8,8,78,8,78,8,8,78,8},
  7/{78,78,78,78,78,78,78,78,78,78,78,78},
  8/{8,78,8,8,78,8,8,8,78,8,8,78}}
  {\foreach \v [count=\c] in \row
     {\fill[siamesegreen!\v!white] ({17+\c},{-\r}) rectangle ++(1,1);}}
\draw[black!45,step=1] (18,-8) grid (30,0);
\draw[black!70,thick] (18,-8) rectangle (30,0);
\node[anchor=south,font=\footnotesize] at (24,0.25) {$M$ (claims)};

\foreach \r/\k in {1/1, 2/3, 3/2, 4/9, 5/5, 6/4, 7/12, 8/4}
  {\fill[labradoramber!85!black] (31,{-\r+0.22}) rectangle ({31+\k},{-\r+0.78});}
\node[anchor=south,font=\footnotesize] at (37,0.25) {$k_i$ (budgets)};

\begin{scope}[yshift=-0.7cm]
\node[anchor=south west] at (0,-10.9) {\textbf{(b)} agreement for one sample $x$ of class $c$};

\foreach \v [count=\c] in {8,78,78,8,78,8,8,78,78,8,8,8,78,8,8,78,8,78,8,8}
  {\fill[sharkblue!\v!white] ({\c-1},-13) rectangle ++(1,1);}
\draw[black!45,step=1] (0,-13) grid (20,-12);
\draw[black!70,thick] (0,-13) rectangle (20,-12);
\node[anchor=west,font=\scriptsize,align=left] at (20.7,-12.5) {selected for $x$: $S(x)$, $|S|=\hat k$};

\foreach \v [count=\c] in {8,78,78,8,8,78,8,78,8,8,78,8,78,8,8,78,8,8,8,78}
  {\fill[siamesegreen!\v!white] ({\c-1},-15) rectangle ++(1,1);}
\draw[black!45,step=1] (0,-15) grid (20,-14);
\draw[black!70,thick] (0,-15) rectangle (20,-14);
\node[anchor=west,font=\scriptsize,align=left] at (20.7,-14.5) {claiming class $c$: $F_c$, from $M_{:,c}$};

\foreach \c in {2,3,8,13,16}
  {\fill[huskyred!55!black,opacity=0.22] ({\c-1},-15.3) rectangle ({\c},-11.7);
   \draw[huskyred!45!black,opacity=0.55] ({\c-1},-15.3) rectangle ({\c},-11.7);}

\draw[{Latex[length=1.6mm]}-{Latex[length=1.6mm]},thick,black!65] (-1.1,-14.6) -- (-1.1,-12.4);

\node[anchor=west,font=\scriptsize] at (0,-16.9)
     {$s_c=\dfrac{|S(x)\cap F_c|}{|S(x)|}=\dfrac{5}{8}$ \qquad the loss $-s_c$ widens the shaded columns};
\end{scope}
\end{tikzpicture}
\caption{\textbf{(a)} Each dictionary coordinate holds a row of trainable class logits $\Lambda_i$ (grey). A trainable vector $\kappa$ splits a \emph{global} association budget $K_{\mathrm{tot}}=\rho d$ into per-feature budgets $k_i$ (orange), and Soft Top-$k_i$ turns each row into a claim $M_i$ carrying exactly $k_i$ units of mass (green). Budgets need not be equal: row~1 is a specialist ($k=1$), row~7 a generalist ($k=12$), and what one row takes another must give up. \textbf{(b)} For a sample of class $c$, the SAE selects $\hat k$ coordinates (blue) while column $M_{:,c}$ marks the coordinates that claimed $c$ (green). The score is the fraction of the selected set that is claimed (shaded); the double arrow indicates that minimizing $-s_c$ pushes from both sides at once, since the encoder must rank claimed features into the top-$\hat k$ while the table must claim features that actually fire. Both panels are drawn in the hard limit; during training $S(x)$ and $F_c$ are the soft weights $\pi(x)$ and $M_{:,c}$.}
\label{fig:method}
\vspace{-0.5cm}
\end{figure}

\paragraph{What the annotation head actually optimizes}\label{sec:lp}
Write the class-balanced form of \Cref{eq:agr} using $\bar A_{ic}=|B_c|^{-1}\sum_{b\in B_c}\pi^{(b)}_i$, the average share of class $c$'s selection budget spent on coordinate $i$. Then $\mathcal{L}_{\mathrm{agr}}=-\tfrac{1}{C}\langle M,\bar A\rangle$, and since each $\pi^{(b)}$ lies on the simplex, $\bar A$ is column-stochastic, so $\mathcal{L}_{\mathrm{agr}}\in[-1,0]$ with $-1$ meaning perfect agreement.

\begin{proposition}[Water-filling]\label{prop:head}
Because $k_i=\sum_c M_{ic}$, the budget variable can be eliminated: the set of tables reachable by \Cref{eq:budget} in the limit is exactly $\mathcal{M}=\{M\in[0,1]^{d\times C}:\sum_{i,c}M_{ic}=K_{\mathrm{tot}}\}$, on which $\langle M,\bar A\rangle$ is linear. Hence for a fixed encoder the optimum is attained at a vertex: $M^{\star}$ is the indicator of the $K_{\mathrm{tot}}$ largest entries of $\bar A$, and $k^{\star}_i=|\{c:\bar A_{ic}>\nu\}|$ for a single global threshold $\nu$ set by $K_{\mathrm{tot}}$.
\end{proposition}
A proof is in \Cref{app:proof}. Three consequences follow.

\emph{(i) The intended behavior is the optimum.} A coordinate claims as many classes as actually spend budget on it, so specialists receive $k_i\approx1$ and shared features $k_i\gg\rho$ with no extra mechanism.

\emph{(ii) The per-row budget is a reparameterization, not a constraint.} A flat variant applying one call to $\mathrm{vec}\,\Lambda$ with budget $K_{\mathrm{tot}}$ has the same feasible set and optimum, so splitting by row costs nothing in principle. We do not ablate this directly: a single SoftTopK call over the flattened $d\times C$ vector is numerically impractical at our scale, so the per-row split is an engineering solution that recovers the same optimum through cheaper, parallel, per-row calls, not a capacity trade-off we verify empirically.

\emph{(iii) Rare features would lose their annotation at the idealized vertex.} Uniformly small rows of $\bar A$ fall below $\nu$ and get $k_i\to0$ at this solution, mirroring the dead-latent pathology $\mathcal{L}_{\mathrm{aux}}$ guards against on the activation side. We are aware of this prediction but do not observe it in practice: annotated rare features remain qualitatively sharp rather than starved. Our hypothesis, not a claim we verify, is that this is because \Cref{prop:head} describes a fixed-encoder optimum against a static $\bar A$, while in joint training $\bar A$ moves with the encoder and $k_i=K_{\mathrm{tot}}\cdot\mathrm{softmax}(\kappa)_i$ is a smooth reparameterization with no pressure to reach its limiting vertex against a non-stationary target; we remain unsure whether this is the full explanation.

\paragraph{Contrastive form} \Cref{eq:agr} rewards the true class but never penalizes the others, so only the budget stops $M$ from drifting toward a uniform matrix. Normalizing over classes repairs this: $\mathcal{L}^{\mathrm{CE}}_{\mathrm{agr}}=\tfrac1B\sum_b[-s^{(b)}_{c_b}/\tau+\log\sum_{c'}\exp(s^{(b)}_{c'}/\tau)]$,
which equals $\mathcal{L}_{\mathrm{agr}}/\tau$ plus a log-sum-exp term and, as $\tau\to\infty$, recovers \Cref{eq:agr} with the class mean subtracted, up to scaling by $\tau$ and an additive constant $\log C$.
Under it $M$ becomes a budget-constrained sparse linear classifier on the SAE code, giving a free evaluation metric (\Cref{sec:eval-classifier}). It is not linear in $M$, so \Cref{prop:head} describes only the base objective; we observe the same specialist/generalist behavior and the absence of annotation collapse under this form as well. Given these properties, we adopt $\mathcal{L}^{\mathrm{CE}}_{\mathrm{agr}}$ as the default agreement loss in all experiments unless stated otherwise (we give the derivation and BatchTopK correction in \Cref{app:proof_pmi}).

\begin{proposition}[PMI equivalence]\label{prop:pmi}
In the hard limit $\alpha_p\to0$ and with support size held constant across samples ($\hat k(x)\equiv\bar k$), the water-filling optimum of \Cref{prop:head} converges to a per-feature PMI threshold: writing $\mathrm{PMI}(i,c)=\log\frac{P(A_i{=}1\mid c)}{P(A_i{=}1)}$ for the hard support indicator $A_i(x)=\1[i\in S(x)]$,
\begin{equation}
M^{\star}_{ic}\;\longrightarrow\;\1\big[\mathrm{PMI}(i,c)>\nu_i\big],
\qquad
\nu_i:=\log(\nu\bar k)-\log P(A_i{=}1)
\label{eq:pmi}
\end{equation}
for the same global threshold $\nu$ as \Cref{prop:head}. That is, the optimal annotation matrix is provably a pointwise-mutual-information thresholding rule, with a per-feature calibrated cutoff $\nu_i$ that falls as feature $i$'s marginal firing rate grows.
\end{proposition}

\section{Experiments}\label{sec:exp}

We evaluate \ourca{} on four questions: does the agreement loss cost reconstruction quality, is the learned annotation $M$ faithful to held-out data, can the SAE and $M$ serve directly as a classifier, and does the encoder itself become better structured for downstream editing?

\paragraph{Setup} All experiments use CLIP ViT-L/14 embeddings of ImageNet-1k \citep{deng2009imagenet}, training on the train split and evaluating on the validation split. We use dictionary size $d=16384$ and association budget $\rho=40$ throughout; remaining hyperparameters are given in \Cref{app:hyperparams}. We denote the three sparsity variants of \Cref{sec:sparsity} as \ourca{} (soft-gradient BatchTopK), \ourca{} Hard, and \ourca{}~$\normlone$. \ourca{} Shuffle is a control that trains the soft-gradient variant with randomly permuted class labels. We train every architecture with 5 seeds and report the mean and standard deviation of each metric across seeds.

\paragraph{Reconstruction Quality}\label{sec:reconstruction}
\begin{table}[t]
\caption{\textbf{Reconstruction quality.} Comparison at matched target $L_0\approx380$. \textbf{$L_0$} is the realized $\normlzero$-norm of the code; \textbf{FVE} is the fraction of input variance explained by the reconstruction; \textbf{Alive} is the fraction of dictionary coordinates that fire on at least one validation sample. Bold marks the best result among architectures matched to the target $L_0$; $^{\dagger}$Top-AFA selects its own $L_0$ and is excluded from this comparison.}
\vspace{-0.5cm}
\label{table:reconstruction}
\begin{center}
\begin{tabular}{llll}
\toprule
Architecture & $L_0$ & FVE $\uparrow$ & Alive $\uparrow$ \\
\midrule
BatchTopK & 381.055$_{\pm 0.073}$ & 0.973$_{\pm 0.000}$ & 0.713$_{\pm 0.013}$ \\
JumpReLU & 383.346$_{\pm 3.158}$ & \textbf{0.986}$_{\pm 0.000}$ & 0.579$_{\pm 0.019}$ \\
Matryoshka & 381.001$_{\pm 0.199}$ & 0.968$_{\pm 0.000}$ & 0.377$_{\pm 0.027}$ \\
Standard & 345.936$_{\pm 1.021}$ & 0.899$_{\pm 0.000}$ & 0.671$_{\pm 0.005}$ \\
Top-AFA$^{\dagger}$ & 1184.456$_{\pm 24.535}$ & 0.991$_{\pm 0.001}$ & 0.662$_{\pm 0.029}$ \\
TopK & 380.000$_{\pm 0.000}$ & 0.970$_{\pm 0.000}$ & 0.455$_{\pm 0.010}$ \\
\midrule
\ourca{} Shuffle & 380.715$_{\pm 0.095}$ & 0.977$_{\pm 0.000}$ & 0.522$_{\pm 0.021}$ \\
\ourca{} & 379.502$_{\pm 0.133}$ & 0.934$_{\pm 0.000}$ & \textbf{1.000}$_{\pm 0.000}$ \\
\ourca{} Hard & 377.967$_{\pm 0.238}$ & 0.910$_{\pm 0.000}$ & \textbf{1.000}$_{\pm 0.000}$ \\
\ourca{} $\normlone$ & 444.028$_{\pm 0.408}$ & 0.864$_{\pm 0.000}$ & \textbf{1.000}$_{\pm 0.000}$ \\
\bottomrule
\end{tabular}
\end{center}
\vspace{-0.5cm}
\end{table}

This is a standard reconstruction/sparsity comparison, included to establish that \ourca{} sits in a reasonable operating range relative to existing SAE architectures. We report fraction of variance explained $\mathrm{FVE}=1-\frac{\sum_j\mathrm{Var}(x_j-\hat x_j)}{\sum_j\mathrm{Var}(x_j)}$,
alongside the realized $L_0$ and the fraction of alive latents in \Cref{table:reconstruction}. Most architectures are matched to the target $L_0\approx380$ directly by their sparsity mechanism; Standard and \ourca{}~$\normlone$ instead reach this target indirectly, via an $\normlone$ coefficient tuned to land near it, which accounts for their larger deviation. Top-AFA is deliberately left unmatched: its premise is choosing its own per-input $L_0$ from the input's $\normltwo$ norm, so constraining it to $\bar k=380$ would defeat the mechanism being evaluated. The agreement loss does lower FVE relative to the unsupervised baselines (0.934 for \ourca{ versus 0.973 for BatchTopK, which shares its forward pass), while \ourca{} Shuffle (0.977) shows that this cost comes from the class signal rather than from the mechanism itself.} Only the \ourca{} variants trained with true labels reach full dictionary utilization (Alive $=1.000$), which we attribute to the affinity matrix giving every coordinate a class to serve.

\paragraph{$M$-Honesty}\label{sec:honesty}
\begin{table}[t]
\caption{\textbf{$M$-honesty.} Mean/median cosine similarity (CS) and precision/recall/F1 between the train-side matrix ($M$ for \ourca{}, post-hoc $\hat M$ for baselines) and the test-side $\hat M$. \ourca{} variants demonstrate that $M$ carries a real, class-faithful association, not just an artifact of training.}
\label{table:honesty}
\label{table:m_honesty}
\begin{center}
\begin{tabular}{@{}llllll@{}}
\toprule
Architecture & Mean CS $\uparrow$ & Median CS $\uparrow$ & Mean Precision $\uparrow$ & Mean Recall $\uparrow$ & Mean F1 $\uparrow$ \\
\midrule
BatchTopK & 0.410$_{\pm 0.009}$ & 0.269$_{\pm 0.043}$ & 0.348$_{\pm 0.007}$ & 0.347$_{\pm 0.008}$ & 0.346$_{\pm 0.008}$ \\
JumpReLU & 0.287$_{\pm 0.010}$ & 0.000$_{\pm 0.000}$ & 0.465$_{\pm 0.016}$ & 0.463$_{\pm 0.016}$ & 0.463$_{\pm 0.016}$ \\
Matryoshka & 0.092$_{\pm 0.002}$ & 0.000$_{\pm 0.000}$ & 0.178$_{\pm 0.018}$ & 0.178$_{\pm 0.017}$ & 0.178$_{\pm 0.017}$ \\
Standard & 0.555$_{\pm 0.004}$ & 0.848$_{\pm 0.000}$ & 0.683$_{\pm 0.004}$ & 0.681$_{\pm 0.004}$ & 0.682$_{\pm 0.004}$ \\
Top-AFA & 0.266$_{\pm 0.016}$ & 0.127$_{\pm 0.048}$ & 0.180$_{\pm 0.005}$ & 0.179$_{\pm 0.005}$ & 0.179$_{\pm 0.005}$ \\
TopK & 0.195$_{\pm 0.003}$ & 0.000$_{\pm 0.000}$ & 0.177$_{\pm 0.004}$ & 0.176$_{\pm 0.004}$ & 0.176$_{\pm 0.004}$ \\
\midrule
\ourca{} Shuffle & 0.017$_{\pm 0.000}$ & 0.000$_{\pm 0.000}$ & 0.034$_{\pm 0.000}$ & 0.042$_{\pm 0.002}$ & 0.009$_{\pm 0.000}$ \\
\ourca{} & 0.822$_{\pm 0.001}$ & 0.827$_{\pm 0.001}$ & 0.766$_{\pm 0.002}$ & \textbf{0.889}$_{\pm 0.001}$ & 0.807$_{\pm 0.001}$ \\
\ourca{} Hard & 0.847$_{\pm 0.001}$ & 0.850$_{\pm 0.001}$ & \textbf{0.835}$_{\pm 0.001}$ & 0.883$_{\pm 0.001}$ & \textbf{0.851}$_{\pm 0.001}$ \\
\ourca{} $\normlone$ & \textbf{0.851}$_{\pm 0.001}$ & \textbf{0.891}$_{\pm 0.001}$ & 0.812$_{\pm 0.001}$ & 0.823$_{\pm 0.000}$ & 0.796$_{\pm 0.001}$ \\
\bottomrule
\end{tabular}
\end{center}
\vspace{-0.5cm}
\end{table}

Here we ask whether the trained $M$ is a faithful description of class structure on unseen data, not just a byproduct of fitting the training set. As a post-hoc reference we compute $\hat M$, a PMI matrix over the same feature-class pairs, with each row's entries below its own $k_i/d$ threshold zeroed out to give it the same row-wise budget structure as $M$. We compare the train-side matrix ($M$ for \ourca{}, $\hat M$ for baselines) against the test-side $\hat M$ computed independently on the ImageNet-1k validation set.

Cosine similarity is computed directly, row-wise, between the train-side and test-side matrices. For precision, recall, and F1, each feature's \emph{claimed} set is the top-$k_i$ classes of its row; ground truth is this set computed on test data, prediction is the same set computed on train data.

As \Cref{table:m_honesty} shows, a \ourca{} variant is the top scorer on every metric: \ourca{}~$\normlone$ on both cosine similarities, \ourca{} Hard on precision and F1, and the soft-gradient \ourca{} on recall. Standard is a surprisingly strong baseline, in particular on median CS, where it exceeds the soft-gradient \ourca{,} but trails all three \ourca{ variants} on every other metric. Most notably, \ourca{} Shuffle collapses to near-zero across the board, confirming that the gain over baselines comes from real class structure and not merely from having a guidance signal at all.

\paragraph{An Out-of-the-Box Classifier}\label{sec:eval-classifier}\label{sec:eval-ca}
\begin{table}[t]
\caption{\textbf{Out-of-the-box classification.} \textbf{Acc}, \textbf{Acc Top~5}, \textbf{Acc Top~10}: top-1, top-5, and top-10 accuracy. \textbf{F1}: macro F1. \textbf{Margin}: gap between the top two class scores. \ourca{~$\normlone$ is the top scorer on every metric, and all \ourca{} variants trained with true labels achieve a larger margin than any baseline. Standard is the strongest baseline: it outperforms the soft-gradient \ourca{} on accuracy and F1 and slightly exceeds \ourca{} Hard on Acc Top~10.}}
\label{table:m_classifier}
\begin{center}
\begin{tabular}{@{}llllll@{}}
\toprule
Architecture & Acc $\uparrow$ & Acc Top 5 $\uparrow$ & Acc Top 10 $\uparrow$ & F1 $\uparrow$ & Margin $\uparrow$ \\
\midrule
BatchTopK & 0.002$_{\pm 0.000}$ & 0.007$_{\pm 0.000}$ & 0.015$_{\pm 0.001}$ & 0.000$_{\pm 0.000}$ & -0.148$_{\pm 0.001}$ \\
JumpReLU & 0.678$_{\pm 0.003}$ & 0.855$_{\pm 0.004}$ & 0.893$_{\pm 0.004}$ & 0.682$_{\pm 0.002}$ & 0.023$_{\pm 0.000}$ \\
Matryoshka & 0.002$_{\pm 0.000}$ & 0.006$_{\pm 0.000}$ & 0.015$_{\pm 0.000}$ & 0.000$_{\pm 0.000}$ & -0.137$_{\pm 0.002}$ \\
Standard & 0.760$_{\pm 0.000}$ & 0.933$_{\pm 0.000}$ & 0.963$_{\pm 0.000}$ & 0.757$_{\pm 0.000}$ & 0.013$_{\pm 0.000}$ \\
Top-AFA & 0.004$_{\pm 0.002}$ & 0.012$_{\pm 0.002}$ & 0.022$_{\pm 0.005}$ & 0.003$_{\pm 0.001}$ & -0.044$_{\pm 0.007}$ \\
TopK & 0.611$_{\pm 0.008}$ & 0.779$_{\pm 0.010}$ & 0.821$_{\pm 0.008}$ & 0.607$_{\pm 0.004}$ & 0.006$_{\pm 0.000}$ \\
\midrule
\ourca{} Shuffle & 0.001$_{\pm 0.000}$ & 0.005$_{\pm 0.001}$ & 0.010$_{\pm 0.001}$ & 0.000$_{\pm 0.000}$ & -0.060$_{\pm 0.002}$ \\
\ourca{} & 0.750$_{\pm 0.002}$ & 0.914$_{\pm 0.003}$ & 0.941$_{\pm 0.001}$ & 0.748$_{\pm 0.003}$ & 0.096$_{\pm 0.001}$ \\
\ourca{} Hard & 0.789$_{\pm 0.001}$ & 0.940$_{\pm 0.000}$ & 0.962$_{\pm 0.000}$ & 0.787$_{\pm 0.001}$ & 0.131$_{\pm 0.001}$ \\
\ourca{} $\normlone$ & \textbf{0.799}$_{\pm 0.001}$ & \textbf{0.955}$_{\pm 0.000}$ & \textbf{0.976}$_{\pm 0.000}$ & \textbf{0.796}$_{\pm 0.001}$ & \textbf{0.168}$_{\pm 0.000}$ \\
\bottomrule
\end{tabular}
\end{center}
\vspace{-0.5cm}
\end{table}

Beyond annotation, the same construction gives \ourca{} a classifier for free: because $M$ is trained to agree with the encoder's own selections, the pair (SAE, $M$) can be read directly as a ready-to-use classifier, with no additional training or fitting required.

Prediction follows the same construction as the agreement loss itself: encode $x$, compute $\pi(x)$, and score each class by $\pi(x)^\top M$ (or $\pi(x)^\top \hat M$ for the post-hoc baselines); the predicted class is the argmax over this score. No separate classifier is fit at any point. With this construction we can report standard classification metrics in \Cref{table:m_classifier}.

\ourca{}~$\normlone$ is the strongest predictor here, and the other two \ourca{ variants trained with true labels follow closely}. The only real competitor is Standard, unsurprising given its strong showing in \Cref{table:honesty}. More interesting is that Standard actually outperforms \ourca{} (the soft-gradient variant) on accuracy and F1, even though it has lower FVE in \Cref{table:reconstruction}. We read this as evidence that reconstruction quality and class separability are, at some point, in direct competition: the soft-gradient \ourca{ trades some classification performance for the best FVE among the \ourca{} variants}, while Standard's comparatively poor reconstruction leaves more of its representational capacity free to encode class structure.

\paragraph{Editing What You Named}\label{sec:tpp_unlearning}
If the agreement loss genuinely reshapes the encoder rather than just labeling it, intervening on a target class should be easier and leak less into others. We test this with the standard TPP benchmark \citep{pmlr-v267-karvonen25a}, at $k\in\{20,200,600\}$ selected features per class, and a custom variant \textbf{TPP PMI} described in \Cref{app:tpp}; a further unlearning task on a joint classifier, using the same intervention, is reported in \Cref{app:unlearn}.

At TPP~20, \ourca{} Shuffle scores best, but this is an artifact of the top-$k$ sensitivity we acknowledge in \Cref{app:tpp} rather than a real advantage: \ourca{} variants simply need more features to represent a class than an unsupervised encoder does. At TPP~600 and TPP~PMI, all three \ourca{} variants lead, and at TPP~200 \ourca{} and \ourca{} Hard lead. \ourca{} Hard and \ourca{}~$\normlone$ achieve this with the lowest edit strength of all architectures, and the soft-gradient \ourca{} with lower edit strength than every baseline except Top-AFA, whose TPP~PMI score is near zero. This indicates a genuinely better-separated encoder.

\begin{table}[t]
\caption{\textbf{Targeted probe perturbation.} TPP at $k\in\{20,200,600\}$ and under the PMI-graded edit (TPP PMI); Edit Strength is the mean $\normlone$ norm of the TPP PMI edit. \ourca{ Hard leads at $k{\geq}200$ and \ourca{}~$\normlone$ under TPP~PMI, with the two smallest edit strengths.}}
\label{table:tpp}
\begin{center}
\begin{tabular}{lccccc}
\toprule
Architecture & TPP 20 $\uparrow$ & TPP 200 $\uparrow$ & TPP 600 $\uparrow$ & TPP PMI $\uparrow$ & Edit Strength $\downarrow$ \\
\midrule
BatchTopK
& 0.450$_{\pm 0.003}$
& 0.266$_{\pm 0.011}$
& 0.051$_{\pm 0.003}$
& 0.277$_{\pm 0.016}$
& 70.850$_{\pm 4.841}$\\

JumpReLU
& 0.102$_{\pm 0.009}$
& 0.463$_{\pm 0.001}$
& 0.281$_{\pm 0.006}$
& 0.453$_{\pm 0.003}$
& 87.082$_{\pm 5.024}$\\

Matryoshka
& 0.431$_{\pm 0.008}$
& 0.113$_{\pm 0.009}$
& 0.015$_{\pm 0.003}$
& 0.104$_{\pm 0.004}$
& 162.848$_{\pm 2.378}$\\

Standard
& 0.428$_{\pm 0.003}$
& 0.465$_{\pm 0.001}$
& 0.432$_{\pm 0.003}$
& 0.430$_{\pm 0.001}$
& 62.379$_{\pm 0.196}$\\

Top-AFA
& 0.174$_{\pm 0.014}$
& 0.464$_{\pm 0.004}$
& 0.356$_{\pm 0.019}$
& 0.012$_{\pm 0.002}$
& 31.282$_{\pm 3.376}$\\

TopK
& 0.423$_{\pm 0.010}$
& 0.304$_{\pm 0.007}$
& 0.067$_{\pm 0.008}$
& 0.309$_{\pm 0.010}$
& 124.047$_{\pm 1.017}$\\
\midrule
\ourca{} Shuffle
& \textbf{0.453}$_{\pm 0.004}$
& 0.272$_{\pm 0.009}$
& 0.052$_{\pm 0.004}$
& 0.001$_{\pm 0.001}$
& 36.936$_{\pm 0.619}$\\

\ourca{}
& 0.082$_{\pm 0.009}$
& 0.471$_{\pm 0.004}$
& 0.513$_{\pm 0.002}$
& 0.462$_{\pm 0.002}$
& 33.816$_{\pm 0.625}$\\

\ourca{} Hard
& 0.064$_{\pm 0.004}$
& \textbf{0.479}$_{\pm 0.002}$
& \textbf{0.517}$_{\pm 0.003}$
& 0.469$_{\pm 0.001}$
& \textbf{23.470}$_{\pm 0.393}$\\

\ourca{} $\normlone$
& 0.011$_{\pm 0.001}$
& 0.394$_{\pm 0.003}$
& 0.474$_{\pm 0.000}$
& \textbf{0.475}$_{\pm 0.001}$
& 23.842$_{\pm 0.121}$\\
\bottomrule
\end{tabular}
\end{center}
\vspace{-0.5cm}
\end{table}

\section{Conclusions}

We introduced \ourca{}, an SAE that automates feature-class association and guides the encoder toward class-separable representations, with three sparsity variants. We derived theoretical properties of this construction, motivating both the architecture and our evaluation design. Empirically, \ourca{} delivers a strong suite of downstream capabilities, including out-of-the-box classification and precise, low-strength intervention. It leads on annotation faithfulness, classification, and TPP at larger $k$ and under PMI-graded editing, at the cost of somewhat lower reconstruction quality and weaker results at small TPP $k$. Across these trade-offs, we recommend \ourca{} Hard as the default variant, as it offers the best balance overall. 
{\bf Limitations} SoftTopK is a major computational bottleneck, limiting scaling to larger dictionaries and to very large class counts, though $2^{14}$ classes is already an unlikely regime in practice.


\bibliographystyle{iclr2027_conference}

\appendix

\section{Proof of \texorpdfstring{\Cref{prop:head}}{Proposition 1}}\label{app:proof}

\begin{proof}
By \Cref{eq:softtopk}, $M_i=\SoftTopK(\Lambda_i,k_i,\alpha_M)$ satisfies $\sum_c M_{ic}=k_i$ and $M_{ic}\in(0,1)$, and as $\Lambda_i$ ranges over $\R^{C}$ the row $M_i$ ranges over the whole relative interior of $\{m\in[0,1]^{C}:\sum_c m_c=k_i\}$. Since $k=K_{\mathrm{tot}}\,\mathrm{softmax}(\kappa)$ ranges over the relative interior of the simplex scaled by $K_{\mathrm{tot}}$, the pair $(\kappa,\Lambda)$ reaches exactly the relative interior of
\[
\mathcal{M}=\Big\{M\in[0,1]^{d\times C}\;:\;\textstyle\sum_{i,c}M_{ic}=K_{\mathrm{tot}}\Big\},
\]
whose closure is $\mathcal{M}$. The objective $M\mapsto\langle M,\bar A\rangle$ is linear and $\mathcal{M}$ is a compact polytope (a box intersected with one hyperplane), so the maximum is attained at a vertex. Vertices of $\mathcal{M}$ have all coordinates in $\{0,1\}$ except at most one, so a maximizer is the indicator of the $K_{\mathrm{tot}}$ largest entries of $\bar A$ (with an arbitrary tie-break, and one fractional entry if $K_{\mathrm{tot}}$ is not an integer). Writing $\nu$ for the $(K_{\mathrm{tot}{+}1)}$-th largest value of $\bar A$ gives $M^{\star}_{ic}=\mathbf{1}[\bar A_{ic}>\nu]$ and hence $k^{\star}_i=\sum_c M^{\star}_{ic}=|\{c:\bar A_{ic}>\nu\}|$.
\end{proof}
Two caveats. The statement holds for a \emph{fixed} encoder; in joint training $\bar A$ moves, and the coupled problem is not linear. It also holds only for the base objective; the contrastive form is not linear in $M$, and softens the vertex solution without changing the water-filling reading.

\section{Proof of \texorpdfstring{\Cref{prop:pmi}}{Proposition 2}}\label{app:proof_pmi}\label{app:pmi}

\begin{proof}
\textbf{Step 1: the hard-limit correspondence.} As $\alpha_p\to0$, $p_i(x)\to A_i(x)=\1[i\in S(x)]$ pointwise, since $\SoftTopK$ converges to the indicator of the selected support (\Cref{eq:softtopk}). With support size held constant, $\hat k(x)\equiv\bar k$, so
\[
\pi_i(x)=\frac{p_i(x)}{\hat k(x)}\;\longrightarrow\;\frac{A_i(x)}{\bar k}.
\]

\textbf{Step 2: $\bar A$ reduces to a rescaled conditional firing rate.} By definition, $\bar A_{ic}=\mathbb E_{x\sim c}[\pi_i(x)]$. Taking the limit of Step 1 inside the expectation,
\[
\bar A_{ic}\;\longrightarrow\;\frac{1}{\bar k}\,\mathbb E_{x\sim c}[A_i(x)] = \frac{P(A_i{=}1\mid c)}{\bar k}.
\]

\textbf{Step 3: recovering PMI.} By definition, $\mathrm{PMI}(i,c)=\log P(A_i{=}1\mid c)-\log P(A_i{=}1)$. Substituting Step 2,
\[
\log\bar A_{ic} = \log P(A_i{=}1\mid c) - \log\bar k = \mathrm{PMI}(i,c) + \log P(A_i{=}1) - \log \bar k,
\]
so
\[
\mathrm{PMI}(i,c) = \log\bar A_{ic} - \log P(A_i{=}1) + \log\bar k. \tag{$\ast$}
\]

\textbf{Step 4: substituting the water-filling threshold.} By \Cref{prop:head}, $M^{\star}_{ic}=\1[\bar A_{ic}>\nu]$ for a single global threshold $\nu$. Equivalently, $\log\bar A_{ic}>\log\nu$. Using $(\ast)$ to rewrite $\log\bar A_{ic}$ in terms of $\mathrm{PMI}(i,c)$,
\begin{align*}
\log\bar A_{ic} > \log\nu
\;\iff\ \\
\mathrm{PMI}(i,c) + \log P(A_i{=}1) - \log\bar k > \log\nu
\;\iff\ \\
\mathrm{PMI}(i,c) > \log(\nu\bar k) - \log P(A_i{=}1).
\end{align*}
Setting $\nu_i:=\log(\nu\bar k)-\log P(A_i{=}1)$ gives exactly
\[
M^{\star}_{ic} \;\longrightarrow\; \1\big[\mathrm{PMI}(i,c)>\nu_i\big],
\]
which is \eqref{eq:pmi}. Since $\nu$ is shared across all features while $P(A_i{=}1)$ is feature-specific, $\nu_i$ is a per-feature calibration of the single global threshold $\nu$, decreasing in $i$'s marginal firing rate.
\end{proof}

\paragraph{Departure from the idealized limit} This correspondence assumes constant support size, $\hat k(x)\equiv\bar k$, which holds for plain per-sample top-$k$ but not for BatchTopK, where $\hat k(x)$ is itself a per-sample random variable. Writing $m_c:=\mathbb E_{x\sim c}[\hat k(x)]$ and expanding $1/\hat k(x)$ to first order around $m_c$,
\[
\bar A_{ic}\approx\frac{P(A_i{=}1\mid c)}{m_c}-m_c^{-2}\,\mathrm{Cov}_{x\sim c}\big(A_i(x),\hat k(x)\big),
\]
so under BatchTopK the correspondence holds only approximately, with a residual governed by $\mathrm{Var}(\hat k\mid c)$ and by any drift of $m_c$ from $\bar k$ across classes. This applies equally to variant (iii): although the $\normlone$-relaxed encoder has no sparse reconstruction, $A(x)$ and $\hat k(x)$ are defined through the same agreement-side $p(x)$ used by every variant (\Cref{sec:sparsity}), so the correspondence and its BatchTopK-style correction apply here as well, with $m_c$ taken over this variant's own $\hat k(x)$.

\section{Hyperparameters}\label{app:hyperparams}

\paragraph{Compute} All architectures except \ourca{} trained in under 10 minutes; \ourca{} variants took approximately 1 hour each, all on a single NVIDIA GH200.

\paragraph{Shared across all architectures} Epochs 50, Learning rate $6\times10^{-4}$, dictionary size $d=16384$, activation dimension $n=768$, auxiliary-loss coefficient $\gamma=1/32$, LR decay start at step 5000, 1000 warmup steps, threshold start at step 1000, dead-feature threshold $2\times10^{6}$ steps, and no $\hat k$-annealing.

\paragraph{\ourca{} variants (Shuffle, \ourca{}, \ourca{} Hard, \ourca{} $\normlone$)} Association budget $\rho=40$, $C=1000$ classes, contrastive temperature $\tau=1$, agreement-loss weight $\mu=0.2$, SoftTopK temperature $\alpha_p=\alpha_M=0.001$, $\tau$-annealing turned off, $\alpha$ annealing turned off. Target $\bar k=380$ for Shuffle, \ourca{}, and \ourca{} Hard; \ourca{} $\normlone$ instead uses an $\normlone$ coefficient of $0.05$ in place of a hard target, with no $\normlone$-annealing schedule. Threshold EMA $\beta=0.999$.

\paragraph{BatchTopK} Target $\bar k=380$; threshold EMA $\beta=0.999$.

\paragraph{JumpReLU} Target $L_0=380$, bandwidth $0.001$, sparsity penalty coefficient $1$, sparsity warmup over 300 steps, threshold EMA $\beta=0.999$.

\paragraph{Matryoshka} Target $\bar k=380$; group fractions $(1/64,\,1/64,\,1/32,\,1/16,\,1/8,\,1/4,\,1/2)$, uniform group weighting; threshold EMA $\beta=0.999$.

\paragraph{Standard} $\normlone$ coefficient $0.025$, sparsity warmup over 300 steps, no feature resampling; threshold EMA $\beta=0.999$.

\paragraph{Top-AFA} $\normlone$ coefficient $0$, AFA coefficient $0.0625$, auxiliary penalty $0.03125$, no auxiliary top-$k$.

\paragraph{TopK} Target $\bar k=380$; threshold EMA $\beta=0.999$.

\section{Targeted Probe Perturbation Setup}\label{app:tpp}

\paragraph{TPP PMI} TPP's top-$k$ selection makes cross-architecture comparison sensitive to the choice of $k$, since an architecture can concentrate its signal into few features and win at small $k$ but lose at large $k$ or vice versa; this is exactly what we observe in \Cref{table:tpp}. We therefore also grade intervention strength continuously rather than thresholding at top-$k$, following \Cref{prop:pmi}: given activation $z$,

\begin{equation}
z' = z\odot(1-\mathrm{strength}), \qquad
\mathrm{strength} = \mathrm{clamp}_{\le1}\!\left(\frac{z}{q_{0.9}(z)}\right)\odot \frac{\mathrm{clamp}_{>0}(\mathrm{PMI}(\cdot,c))}{q_{0.99}\big(\mathrm{clamp}_{>0}(\mathrm{PMI}(\cdot,c))\big)}\cdot\eta,
\end{equation}

with $q_{0.9}(z)$ a per-feature activation quantile, $q_{0.99}(\cdot)$ the 99th percentile used to normalize the clamped PMI scores, and $\eta$ a strength multiplier, so common low-activation features are left largely untouched while strongly firing, class-informative features are targeted. \textbf{TPP PMI} is TPP evaluated with this edit in place of top-$k$ selection; \textbf{Edit Strength} is the mean $\normlone$ norm of the strength vector.

\paragraph{TPP setup and parameters} We evaluate 50 concepts (sampled at random rather than fixed by index, but the same for each architecture), with a balanced negative-to-positive ratio of $1.0$ for each concept's binary partition. Probes are trained for 20 epochs, learning rate $10^{-3}$, weight decay $10^{-4}$, batch size 256.

Asociation budget $\rho=40$, edit-strength scale $5.0$ (up from a default of $1.0$), PMI upper percentile $99.0$, and activation-scale percentile $90.0$. For \ourca{} variants, we use the model's own $M$ and native selection weights directly rather than an empirical PMI pass; all other architectures use the empirical PMI pass by default.

Data are processed in batches of 4096 with chunk size 2048 and 4 worker processes; we use the full train and test activation sets (no cap on the number of examples) and seed $0$.

We use the same setting for the Unlearning Benchmark in the next section.

\section{Custom Unlearning Benchmark}\label{app:unlearn}

\begin{table}[h]
\caption{\textbf{Joint-classifier unlearning.} Graded-vs-shuffled forget-accuracy drop, retain-accuracy drop, and edit strength (mean $\normlone$ norm); the two drops are measured relative to baseline classifier accuracy. \ourca{} variants dominate on forget drop while keeping edit strength small.}
\label{table:unlearn}
\begin{center}
\begin{tabular}{lrll}
\toprule
Architecture & Graded vs. Shuffled Forget Drop $\uparrow$ & Retain Drop $\downarrow$ & Edit Strength $\downarrow$ \\
\midrule
BatchTopK & -0.284$_{\pm 0.010}$ & 0.075$_{\pm 0.009}$ & 70.850$_{\pm 4.841}$ \\
JumpReLU & 0.170$_{\pm 0.017}$ & 0.034$_{\pm 0.003}$ & 87.082$_{\pm 5.024}$ \\
Matryoshka & 0.000$_{\pm 0.000}$ & 0.421$_{\pm 0.014}$ & 162.848$_{\pm 2.378}$ \\
Standard & 0.224$_{\pm 0.046}$ & 0.055$_{\pm 0.000}$ & 62.379$_{\pm 0.196}$ \\
Top-AFA & -0.442$_{\pm 0.032}$ & \textbf{0.001}$_{\pm 0.000}$ & 31.282$_{\pm 3.376}$ \\
TopK & 0.000$_{\pm 0.000}$ & 0.199$_{\pm 0.006}$ & 124.047$_{\pm 1.017}$ \\
\midrule
\ourca{} Shuffle & 0.016$_{\pm 0.014}$ & 0.032$_{\pm 0.001}$ & 36.936$_{\pm 0.619}$ \\
\ourca{} & \textbf{0.816}$_{\pm 0.008}$ & 0.023$_{\pm 0.002}$ & 33.816$_{\pm 0.625}$ \\
\ourca{} Hard & \textbf{0.816}$_{\pm 0.009}$ & 0.011$_{\pm 0.002}$ & \textbf{23.470}$_{\pm 0.393}$ \\
\ourca{} $\normlone$ & 0.803$_{\pm 0.001}$ & 0.002$_{\pm 0.000}$ & 23.842$_{\pm 0.121}$ \\
\bottomrule
\end{tabular}
\end{center}
\end{table}

As a complement to the targeted probe perturbation results in \Cref{sec:tpp_unlearning}, we evaluate the same PMI-graded intervention on a joint multi-class classifier rather than independent per-class probes, testing whether selective ablation still works when classes must be distinguished by a single shared decision boundary rather than $C$ separate ones.

We train a linear classifier on SAE-reconstructed embeddings and record its baseline forget and retain accuracy. We then apply two edits: the graded intervention of \Cref{sec:tpp_unlearning}, and the same intervention with PMI scores shuffled among the features that had nonzero PMI for the target class, isolating whether the intervention's quality comes from targeting the right subset of features or from the fine-grained ranking within it. \Cref{table:unlearn} reports the drop in forget accuracy under the graded edit relative to baseline, minus the same drop under the shuffled edit (\textbf{Graded vs.\ Shuffled Forget Drop}); the drop in retain accuracy relative to baseline (\textbf{Retain Drop}); and \textbf{Edit Strength}, the mean $\normlone$ norm of the strength vector, as in \Cref{sec:tpp_unlearning}.

On the forget-drop metric, all three \ourca{ variants trained with true labels} dominate by a wide margin. Here, shuffling PMI scores within the selected subset sharply degrades the edit, showing that \ourca{}'s selected features are not interchangeable: the fine-grained ranking within the subset, not just membership in it, is what drives the intervention. Baselines like JumpReLU tell the opposite story: they remain effective even after shuffling, meaning their selected features are largely interchangeable and carry little differentiated class information individually. A negative value, as for BatchTopK and Top-AFA, means the shuffled edit achieved a \emph{larger} forget-accuracy drop than the graded one. On retain drop, \ourca{} variants are best apart from Top-AFA, which tops this column but is worst on forget drop. Edit strength stays small across all \ourca{} variants, consistent with \Cref{table:tpp}. Together with the main-text results, this confirms that precise, low-strength intervention extends from independent probes to a joint classifier as well.

\section{Qualitative Example: Patch-Token Class Alignments}\label{app:dino}

As a further demonstration of \ourca{}'s applicability beyond the CLIP/ImageNet setting used throughout the paper, we train \ourca{} on patch tokens from DINOv3 (\texttt{facebook/dinov3-vitb16-pretrain-lvd1689m}), assigning every patch token the class label of its parent image rather than any finer-grained, per-patch segmentation label. Despite this deliberately naive labeling scheme, in which a single whole-image class is spread uniformly across all of an image's patch tokens, we find that \ourca{} still learns individual features that fire specifically on the image region corresponding to their claimed class.

\Cref{fig:qualitative} shows two examples on an image containing both a Siamese cat and an American Staffordshire terrier. The top row shows features claimed by ``Siamese cat, Siamese,'' each activating on the cat rather than the dog, with sharper spatial localization for the lower-budget, more specialist feature (feat.\ 1277, $k_i=1.0$) than for the higher-budget, more generalist one (feat.\ 15, $k_i=8.0$). The bottom row shows the analogous pattern for features claimed by ``American Staffordshire terrier, Staffordshire terrier, American pit bull terrier, pit bull terrier,'' localizing on the dog. This is a further example of how \ourca{} can be used in practice, extracting spatially meaningful, class-aligned features even from a naive, image-level labeling scheme with no segmentation supervision.

\end{document}